\documentclass{article}
\usepackage{tikz}
\usepackage{xcolor}

\usepackage[preprint]{neurips_2025}

\usepackage[utf8]{inputenc}
\usepackage[T1]{fontenc}
\usepackage{hyperref}
\usepackage{caption}
\usepackage{amsmath}
\usepackage{amssymb}
\usepackage{url}
\usepackage{booktabs}
\usepackage{amsfonts}
\usepackage{nicefrac}
\usepackage{microtype}
\usepackage{xcolor}
\usepackage{amsthm}
\usepackage{graphicx}
\usepackage{algorithm}
\usepackage{algpseudocode}
\usepackage{float}

\newtheorem{theorem}{Theorem}
\newtheorem{lemma}[theorem]{Lemma}
\newtheorem{definition}{Definition}
\newtheorem{assumption}{Assumption}

\title{SafeMed-R1: Adversarial Reinforcement Learning for Generalizable and Robust Medical Reasoning in Vision--Language Models}

\author{%
  A.A. Gde Yogi Pramana$^{*}$ \quad Jason Ray \quad Anthony Jaya \quad Michael Wijaya \\
  National University of Singapore \\
  \texttt{\{jason.ray,anthony.jaya,e1554354\}@u.nus.edu} \\
  \texttt{yogi.p@comp.nus.edu.sg}
}

\begin{document}

\maketitle

\begin{abstract}
  Vision--Language Models (VLMs) show significant promise for Medical Visual Question Answering (VQA), yet their deployment in clinical settings is hindered by severe vulnerability to adversarial attacks. Standard adversarial training, while effective for simpler tasks, often degrades both generalization performance and the quality of generated clinical reasoning. We introduce SafeMed-R1, a hybrid defense framework that ensures robust performance while preserving high-quality, interpretable medical reasoning. SafeMed-R1 employs a two-stage approach: at training time, we integrate Adversarial Training with Group Relative Policy Optimization (AT-GRPO) to explicitly robustify the reasoning process against worst-case perturbations; at inference time, we augment the model with Randomized Smoothing to provide certified $L_2$-norm robustness guarantees. We evaluate SafeMed-R1 on the OmniMedVQA benchmark across eight medical imaging modalities comprising over 88,000 samples. Our experiments reveal that standard fine-tuned VLMs, despite achieving 95\% accuracy on clean inputs, collapse to approximately 25\% under PGD attacks. In contrast, SafeMed-R1 maintains 84.45\% accuracy under the same adversarial conditions, representing a 59 percentage point improvement in robustness. Furthermore, we demonstrate that models trained with explicit chain-of-thought reasoning exhibit superior adversarial robustness compared to instruction-only variants, suggesting a synergy between interpretability and security in medical AI systems.
\end{abstract}

\section{Introduction}

The integration of Vision--Language Models (VLMs) into healthcare promises to revolutionize clinical workflows, particularly in Medical Visual Question Answering (VQA). These systems interpret complex medical images ($I$) to answer clinical questions ($Q$), facilitating diagnostic support and improving accessibility to medical expertise \cite{li2023llavamedtraininglargelanguageandvision, moor2023medflamingomultimodalmedicalfewshot}. However, the deployment of deep learning models in safety-critical domains is severely hampered by their susceptibility to adversarial attacks \cite{goodfellow2015explaining, carlini2017towards}.

In the medical context, this vulnerability is particularly alarming. Subtle, often imperceptible perturbations to input images whether maliciously crafted or arising from variations in imaging acquisition, can cause VLMs to generate drastically incorrect answers and misleading clinical reasoning \cite{finlayson2019adversarial, ma2021understanding}. Such failures undermine clinical trust and pose significant risks to patient safety, highlighting an urgent need for robust and reliable medical AI systems \cite{owasp-mlsec-top10}.

Traditional defense mechanisms, such as standard Adversarial Training (AT) \cite{madry2018towards}, aim to improve empirical robustness by optimizing against worst-case perturbations. While effective in simpler classification tasks, applying AT to complex reasoning models like VLMs presents unique challenges. In particular, AT can degrade the model's generalization capabilities on clean data and may compromise the coherence and quality of the generated reasoning, which is critical for clinical interpretability \cite{tsipras2019robustness}. Furthermore, empirical defenses offer no formal guarantees of stability.

To address these challenges, we introduce SafeMed-R1, a hybrid defense framework designed to ensure both robust performance and high-quality, generalizable medical reasoning. SafeMed-R1 employs a two-stage approach:
\begin{enumerate}
  \item \textbf{Training Time (Empirical Defense):} We integrate Adversarial Training with Group Relative Policy Optimization (GRPO) \cite{lai2025medr1reinforcementlearninggeneralizable, shao2024deepseekmathpushinglimitsmathematical}. GRPO is a stable reinforcement learning algorithm that explicitly optimizes the VLM for the quality of its reasoning trace. By applying AT within the GRPO framework (AT-GRPO), we robustify the reasoning process itself against worst-case multimodal perturbations, mitigating the degradation in reasoning quality typically seen with standard AT.
  \item \textbf{Inference Time (Certified Defense):} We augment the adversarially trained VLM with Randomized Smoothing (RS) \cite{cohen2019certified}. RS provides mathematical guarantees that the model's prediction remains stable within a defined $L_2$-norm neighborhood around the input, offering a layer of certified robustness crucial for clinical deployment.
\end{enumerate}

Our contributions are threefold: (1) we propose SafeMed-R1, the first framework to combine adversarial reinforcement learning (AT-GRPO) with certified defenses (RS) for medical VQA; (2) we show that integrating AT with GRPO significantly enhances robustness against PGD attacks while preserving high-quality clinical reasoning; and (3) we provide a comprehensive evaluation across eight medical modalities, demonstrating that SafeMed-R1 outperforms existing models in adversarial settings.

\section{Related Works}

Our work lies at the intersection of medical VLMs, reinforcement learning for reasoning, and adversarial robustness.

\paragraph{Medical Vision--Language Models.}
The development of specialized VLMs for the medical domain has accelerated with the availability of large-scale datasets. Early efforts focused on adapting general-domain architectures through domain-specific pretraining and instruction tuning. PMC-VQA \cite{zhang2024pmcvqavisualinstructiontuning} utilized data from PubMed Central to align medical images and text. LLaVA-Med \cite{li2023llavamedtraininglargelanguageandvision} demonstrated the effectiveness of low-cost instruction tuning using GPT-4-generated data on clinical figure, caption pairs. Med-Flamingo \cite{moor2023medflamingomultimodalmedicalfewshot} explored few-shot learning using interleaved medical data, highlighting alternatives when labeled data are scarce. While these models establish the feasibility of medical VQA, they primarily optimize accuracy on clean data and do not directly address adversarial vulnerabilities.

\paragraph{Reinforcement Learning for Clinical Reasoning.}
To move beyond imitation learning and enhance the logical coherence of VLM outputs, recent work has increasingly adopted reinforcement learning (RL). RL enables optimization of non-differentiable objectives, such as reasoning quality evaluated by external reward models \cite{ouyang2022training}. Med-VLM-R1 \cite{pan2025medvlmr1incentivizingmedicalreasoning} applies task-specific reward optimization to reduce hallucinations. Med-R1 \cite{lai2025medr1reinforcementlearninggeneralizable} introduces GRPO, a stable policy-gradient method, showing improvements in generalization and reasoning capabilities. RARL \cite{pham2025rarlimprovingmedicalvlm} combines RL with parameter-efficient tuning (LoRA). SafeMed-R1 builds on these advances by integrating GRPO within an adversarial training framework, explicitly optimizing reasoning quality under worst-case conditions.

\paragraph{Adversarial Robustness.}
The susceptibility of neural networks to adversarial examples, inputs perturbed slightly to induce misclassification, is well documented \cite{szegedy2014intriguing}. Common attack methods include the single-step Fast Gradient Sign Method (FGSM) \cite{goodfellow2015explaining}, iterative Projected Gradient Descent (PGD) \cite{madry2018towards}, and optimization-based attacks such as Carlini \& Wagner (C\&W) \cite{carlini2017towards}.

The predominant empirical defense is Adversarial Training (AT) \cite{madry2018towards}, which trains models on adversarially generated examples. While effective, AT often induces a trade-off between robustness and clean accuracy \cite{tsipras2019robustness}. To obtain formal guarantees, certified defenses have been proposed. Randomized Smoothing (RS) \cite{lecuyer2019certified, cohen2019certified} has emerged as a scalable method for providing $L_2$-norm certified robustness by evaluating the model's consensus over a noise-perturbed input distribution. In the medical domain, prior work has shown that imaging models are highly vulnerable to attacks \cite{finlayson2019adversarial}, but the robustness of medical VQA reasoning remains under-explored. SafeMed-R1 addresses this gap by combining AT (integrated with RL) and RS to provide a more comprehensive defense strategy for medical VLMs.

\section{Problem Formulation}

Our primary objective is to optimize a Vision--Language Model (VLM) for the demanding task of Medical Visual Question Answering (VQA). In this high-stakes domain, the model must not only deliver factually correct answers but also exhibit high-quality clinical reasoning and adhere to a precise output format. To achieve this comprehensive optimization, we employ a sequential two-stage training process: Supervised Fine-Tuning (SFT) followed by targeted Reinforcement Learning using Group Relative Policy Optimization (GRPO).

\paragraph{Task definition and notation}

We formalize the VQA task as follows. The model processes a multimodal input state $s=(I, Q)$, where $I$ represents a medical image and $Q$ is a clinical question. The VLM, parameterized by $\theta$, defines a policy $\pi_{\theta}(Y \mid s)$ that models the probability distribution over all possible output sequences $Y$.

\paragraph{Structured output format (Chain-of-Thought)}

To ensure interpretability and facilitate the evaluation of reasoning, we enforce a specific structure on the desired output $Y$. This structure must contain both the internal reasoning process (the thought trace $T$) and the final conclusion $A$:
\begin{equation}
  Y = \text{"<think>"} \circ T \circ \text{"</think><answer>"} \circ A \circ \text{"</answer>"},
\end{equation}
where $\circ$ denotes string concatenation. We also define an extraction function $\text{ExtractT}(Y) \rightarrow T$ that parses the reasoning steps $T$ from the generated sequence $Y$ \cite{wei2023chainofthoughtpromptingelicitsreasoning}.

\subsection{Stage 1: Supervised Fine-Tuning (SFT)}

The first stage lays the foundation for the model's behavior. The objective of SFT is to establish baseline correctness and strictly enforce the required Chain-of-Thought output format by leveraging high-quality demonstrations.

\paragraph{Data preparation}

This stage relies on a supervised dataset $\mathcal{D}_{SFT} = \{(s_i, Y^*_i)\}_{i=1}^N$. Crucially, $Y^*_i$ is the ground-truth output, already correctly formatted and containing both expert-level reasoning $T^*$ and the correct answer $A^*$.

\paragraph{Objective: Maximum Likelihood Estimation (MLE)}

The goal of SFT is to adjust the parameters $\theta$ such that the model maximizes the likelihood of mimicking the ground-truth sequences present in the dataset:
\begin{equation}
  \theta_{SFT}
  = \arg\max_{\theta} \sum_{i=1}^N \log \pi_{\theta}(Y^*_i \mid s_i).
\end{equation}

\paragraph{Loss function (cross-entropy)}

Since the VLM generates the sequence autoregressively, we model the probability of a sequence $Y^*=(y_1^*, \dots, y_L^*)$ using the chain rule:
\begin{equation}
  \pi_{\theta}(Y^* \mid s)
  = \prod_{t=1}^{L} P_\theta(y_t^* \mid s, y_{<t}^*).
\end{equation}
The optimization is therefore framed as minimizing the negative log-likelihood, resulting in the standard cross-entropy loss:
\begin{equation}
  \label{eq:sft_loss}
  \mathcal{L}_{SFT}(\theta)
  = - \mathbb{E}_{(s, Y^*) \sim \mathcal{D}_{SFT}}
    \left[ \sum_{t=1}^{L} \log P_\theta(y_t^* \mid s, y_{<t}^*) \right].
\end{equation}

Upon completion of this stage, the resulting policy $\pi_{\theta_{SFT}}$ reliably generates factually accurate answers and adheres to the required format. However, SFT optimizes for imitation, not necessarily the intrinsic quality or coherence of the reasoning itself. Thus, $\pi_{\theta_{SFT}}$ serves as the essential initialization for the next optimization phase.

\subsection{Stage 2: Group Relative Policy Optimization (GRPO)}

Having established the foundation, the objective of Stage 2 is to move beyond imitation and specifically enhance the \emph{quality} of the reasoning trace $T$—its coherence, logical flow, and clinical soundness. This requires a pivot from supervised learning to a reinforcement learning framework, initialized with $\theta_{SFT}$.

\paragraph{RL formulation}

We model the VQA task as a contextual bandit (a single-step MDP), where the state is the input $s=(I, Q)$, the action is the generation of the entire sequence $Y$, and the policy is the VLM $\pi_\theta$.

\paragraph{The reasoning reward model (RM)}

To optimize for reasoning quality, we need a mechanism to evaluate it. We introduce a specialized reward model (RM), $R_\phi$, parameterized by $\phi$. This RM is trained independently (e.g., on human preferences for clinical reasoning) to assess the quality of the thought trace $T$, irrespective of the final answer's correctness.

\begin{definition}[Reasoning quality reward]
\label{def:reasoning_quality_reward}
$R_\phi(T \mid s)$ provides a scalar score evaluating the coherence, logical structure, and clinical soundness of the trace $T$ given the context $s$.
\end{definition}

The reward signal $r(Y, s)$ during RL training is therefore derived by applying the RM specifically to the extracted reasoning segment:
\begin{equation}
  \label{eq:reward_definition}
  r(Y, s) = R_\phi(\text{ExtractT}(Y) \mid s).
\end{equation}

\paragraph{The RL objective}

With the reward signal defined, the goal is now to maximize the expected reward:
\begin{equation}
  \label{eq:rl_objective}
  J(\theta)
  = \mathbb{E}_{s \sim \mathcal{D}}
    \mathbb{E}_{Y \sim \pi_{\theta}(Y \mid s)} [r(Y, s)].
\end{equation}

\paragraph{Optimization via GRPO}

We utilize GRPO, a stable policy gradient algorithm, to optimize $J(\theta)$. The standard policy gradient theorem gives:
\begin{equation}
  \label{eq:policy_gradient}
  \nabla_\theta J(\theta)
  = \mathbb{E}_{s, Y \sim \pi_\theta} [r(Y, s)\, \nabla_\theta \log \pi_\theta(Y \mid s)].
\end{equation}
To reduce variance, we introduce an advantage function $A(Y, s) = r(Y, s) - V(s)$ and rewrite:
\begin{equation}
  \label{eq:advantage_gradient}
  \nabla_\theta J(\theta)
  = \mathbb{E}[A(Y, s)\, \nabla_\theta \log \pi_\theta(Y \mid s)].
\end{equation}

GRPO avoids a learned critic by normalizing rewards within a group of $K$ samples:
\begin{equation}
  \label{eq:grpo_sampling}
  \{Y_i\}_{i=1}^K \sim \pi_{\theta_{old}}(\cdot \mid s),
\end{equation}
\begin{equation}
  \label{eq:grpo_stats}
  \bar{r} = \frac{1}{K}\sum_{i=1}^K r_i, \quad
  \text{Var}(r) = \frac{1}{K-1}\sum_{i=1}^K (r_i - \bar{r})^2,
\end{equation}
\begin{equation}
  \label{eq:grpo_advantage}
  \hat{A}_i = \frac{r_i - \bar{r}}{\sqrt{\text{Var}(r) + \epsilon_{std}}}.
\end{equation}

We define the importance sampling ratio:
\begin{equation}
  \label{eq:importance_sampling}
  \rho_i(\theta)
  = \frac{\pi_\theta(Y_i \mid s)}{\pi_{\theta_{old}}(Y_i \mid s)}.
\end{equation}
Then the GRPO objective is:
\begin{equation}
  \label{eq:grpo_loss}
  \mathcal{L}^{GRPO}(\theta)
  = \mathbb{E} \left[
    \min\left(
      \rho_i(\theta)\, \hat{A}_i,\,
      \text{clip}(\rho_i(\theta), 1 - \epsilon_{clip}, 1 + \epsilon_{clip})\, \hat{A}_i
    \right)
  \right].
\end{equation}

The resulting model $\pi_{\theta^*}$ successfully retains the formatting and baseline correctness learned during SFT, while the reasoning segments $T$ are now explicitly optimized for coherence, logical structure, and clinical soundness as defined by the reward model $R_\phi$ \cite{shao2024deepseekmathpushinglimitsmathematical}.

\section{Method}

\subsection{General Overview}

\begin{figure}[H]
  \centering
  \includegraphics[width=\textwidth]{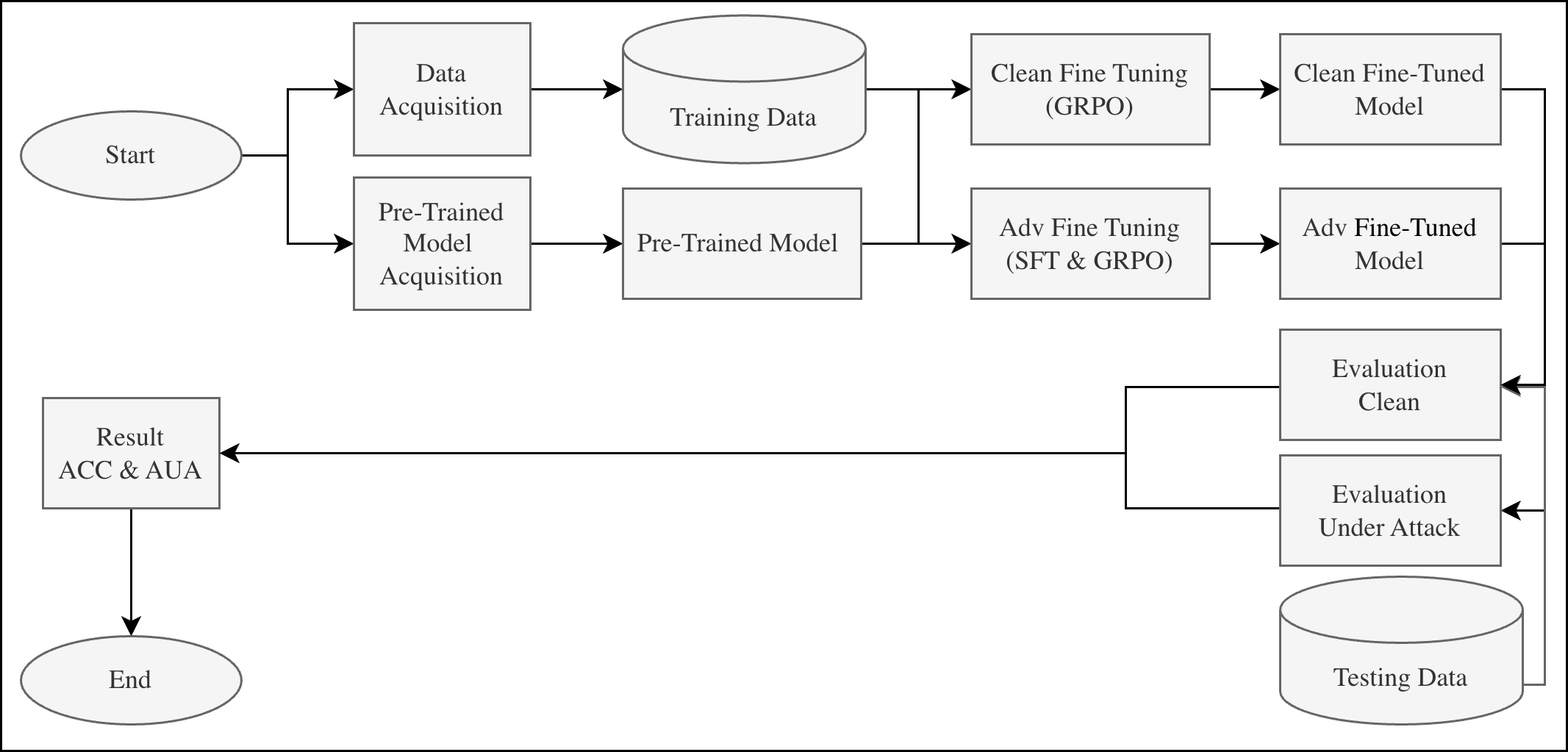}
  \caption{Method general overview.}
  \label{fig:general-overview}
\end{figure}

Figure~\ref{fig:general-overview} summarizes our workflow. We first collect downstream task data and construct training and test splits (\emph{Data Acquisition} $\rightarrow$ \emph{Training/Testing Data}), while in parallel acquiring an off-the-shelf vision--language checkpoint as the base \emph{Pre-Trained Model}. From this checkpoint we train two variants: a \emph{clean} model obtained via GRPO on clean training data (\emph{Clean Fine-Tuning (GRPO)} $\rightarrow$ \emph{Clean Fine-Tuned Model}) and an \emph{adversarial} model obtained via SFT+GRPO on a mixture of clean and adversarially perturbed samples (\emph{Adv Fine-Tuning (SFT \& GRPO)} $\rightarrow$ \emph{Adv Fine-Tuned Model}). Both models are evaluated on the held-out \emph{Testing Data}, measuring performance on clean inputs (\emph{Evaluation Clean}) and under a suite of adversarial attacks (\emph{Evaluation Under Attack}). We report standard clean accuracy (ACC) and our robustness metric AUA (area under the accuracy-under-attack curve), enabling a direct comparison of how each training strategy trades off task performance and robustness.

\subsection{Data and Model Acquisition}

\subsubsection{Data acquisition}

We use OmniMedVQA \cite{hu2024omnimedvqa}, a comprehensive medical-domain VQA benchmark compiled from 73 publicly available medical datasets. From this corpus, we select over 88{,}000 image question pairs across 8 different modalities, which we split into training and testing sets, detailed splits are provided in Table \ref{tab:dataset_splits}.

\begin{table} [H]
    \caption{Dataset Splits (Train:Test = 80:20)\label{tab:dataset_splits}}
    \centering
    \begin{tabular}{lrrr}
    \toprule
    \textbf{Modality} & \textbf{Test} & \textbf{Train} & \textbf{$\sum$} \\
    \midrule
    CT (Computed Tomography) & 3241 & 12567 & 15808 \\
    Dermoscopy & 1306 & 5373 & 6679 \\
    Fundus Photography & 1098 & 4300 & 5398 \\
    Microscopy Images & 1110 & 4570 & 5680 \\
    MRI (Magnetic Resonance Imaging) & 6370 & 25507 & 31877 \\
    OCT (Optical Coherence Tomography) & 848 & 3798 & 4646 \\
    Ultrasound & 2074 & 8917 & 10991 \\
    X-Ray & 1615 & 6301 & 7916 \\
    \midrule
    \textbf{$\sum$} & \textbf{17662} & \textbf{71333} & \textbf{88995} \\
    \bottomrule
    \end{tabular}
\end{table}

\subsubsection{Model acquisition}

We obtain our pre-trained models from Hugging Face. Specifically we use Qwen2.5-VL 3B and 7B \cite{qwen2025qwen25technicalreport}, and Qwen3-VL 4B and 8B \cite{yang2025qwen3technicalreport} as our base vision--language architectures. These models were selected due to their strong multimodal reasoning capabilities, open availability, and compatibility with both LoRA and GRPO finetuning frameworks. All checkpoints are downloaded directly through the official Hugging Face repositories to ensure consistency and reproducibility.

\subsection{Fine-Tuning Methodology}

We optimize the VLM policy $\pi_\theta$ for Medical VQA using a two-stage procedure: Supervised Fine-Tuning (SFT) followed by Reinforcement Learning with Group Relative Policy Optimization (GRPO). The first stage targets answer correctness and adherence to the required output format, while the second stage explicitly improves the quality and coherence of clinical reasoning. We compare the standard setting (Clean Fine-Tuning) against our robust variant (Adversarial Fine-Tuning) within the SafeMed-R1 framework.

\subsubsection{Clean Fine-Tuning}

The standard fine-tuning pipeline maximizes performance on the distribution of clean medical inputs. It follows the two-stage process in Section~3, using the SFT loss in Eq.~\eqref{eq:sft_loss} and the GRPO objective in Eq.~\eqref{eq:grpo_loss}, trained solely on the original dataset $\mathcal{D}$. The procedure is summarized in Algorithm~\ref{alg:clean_ft}.

\begin{algorithm}[H]
\caption{Clean Fine-Tuning (SFT + GRPO)}
\label{alg:clean_ft}
\begin{algorithmic}[1]
\Require VLM parameters $\theta$, reward model $R_\phi$, datasets $\mathcal{D}_{\mathrm{SFT}}, \mathcal{D}_{\mathrm{RL}}$
\State \textbf{Stage 1: SFT}
\State Optimize $\theta$ by minimizing $\mathcal{L}_{\mathrm{SFT}}(\theta)$ (Eq.~\eqref{eq:sft_loss}) on $\mathcal{D}_{\mathrm{SFT}}$
\State $\theta_{\mathrm{SFT}} \leftarrow \theta$
\State \textbf{Stage 2: GRPO} \hfill (initialize RL from SFT)
\State $\theta \leftarrow \theta_{\mathrm{SFT}}$
\For{iteration $j = 1, 2, \dots$}
  \State $\pi_{\theta_{\mathrm{old}}} \leftarrow \pi_\theta$
  \For{minibatch $B \subset \mathcal{D}_{\mathrm{RL}}$}
    \For{state $s \in B$}
      \State Sample $\{Y_i\}_{i=1}^K \sim \pi_{\theta_{\mathrm{old}}}(\cdot \mid s)$
      \State Compute rewards $\{r_i\}$ using $R_\phi$ (Eq.~\eqref{eq:reward_definition})
      \State Estimate advantages $\{\hat{A}_i\}$ via Eq.~\eqref{eq:grpo_advantage}
    \EndFor
  \EndFor
  \State Update $\theta$ by maximizing $\mathcal{L}^{\mathrm{GRPO}}(\theta)$ (Eq.~\eqref{eq:grpo_loss})
\EndFor
\end{algorithmic}
\end{algorithm}

\subsubsection{Adversarial Fine-Tuning (SafeMed-R1)}

To induce robustness against input perturbations, SafeMed-R1 integrates adversarial training (AT) into both the SFT and GRPO stages. AT recasts learning as a robust minimax problem:
\begin{equation}
  \label{eq:minimax_obj}
  \theta^*
  = \arg\min_{\theta} \mathbb{E}_{s}
  \left[
    \max_{\delta \in \Delta(s)}
    \mathcal{L}\bigl(\theta, s + \delta\bigr)
  \right],
\end{equation}
where $\Delta(s)$ is the threat set, and $\mathcal{L}$ denotes the stage-specific objective (either $\mathcal{L}_{\mathrm{SFT}}$ or the negative expected reward $-J(\theta)$).

\paragraph{Inner Maximization (Attack Generation).}

The inner maximization searches for a worst-case perturbation $\delta^*$. We approximate this using Projected Gradient Descent (PGD), applied primarily to the visual modality $I$. At iteration $k$,
\begin{equation}
  \label{eq:pgd_update}
  \delta^{k+1}
  = \Pi_{\Delta(I)}\!\left(
      \delta^k
      + \alpha \cdot \mathrm{sign}
      \bigl(
        \nabla_\delta \mathcal{L}\bigl(\theta_{\mathrm{old}}, (I + \delta^k, Q)\bigr)
      \bigr)
    \right),
\end{equation}
where $\alpha$ is the step size and $\Pi_{\Delta(I)}$ projects back onto the $L_p$-ball $\Delta(I)$. This yields the adversarial state $s'_{\mathrm{adv}} = s + \delta^*$.

\paragraph{Implementation Details.}

During AT-SFT (Stage 1), PGD maximizes $\mathcal{L}_{\mathrm{SFT}}$. During AT-GRPO (Stage 2), the adversary minimizes the expected reward $J(\theta)$, which is equivalent to maximizing $\mathcal{L} = -J(\theta)$. The outer loop then updates $\theta$ on these adversarial examples. The full SafeMed-R1 procedure is given in Algorithm~\ref{alg:adv_ft}.

\begin{algorithm}[H]
\caption{Adversarial Fine-Tuning (SafeMed-R1: AT-SFT + AT-GRPO)}
\label{alg:adv_ft}
\begin{algorithmic}[1]
\Require VLM parameters $\theta$, reward model $R_\phi$, PGD parameters $(\alpha, N_{\mathrm{PGD}}, \epsilon)$
\State \textbf{Stage 1: AT-SFT}
\While{not converged}
  \For{minibatch $B \subset \mathcal{D}_{\mathrm{SFT}}$}
    \State $B'_{\mathrm{adv}} \leftarrow \emptyset$
    \For{$(s, Y^*) \in B$} \Comment{inner maximization: maximize SFT loss}
      \State Generate $\delta^*$ using PGD (Eq.~\eqref{eq:pgd_update}) with $\mathcal{L} = \mathcal{L}_{\mathrm{SFT}}$
      \State $s'_{\mathrm{adv}} \leftarrow s + \delta^*$
      \State Add $(s'_{\mathrm{adv}}, Y^*)$ to $B'_{\mathrm{adv}}$
    \EndFor
    \State Update $\theta$ by minimizing $\mathcal{L}_{\mathrm{SFT}}(\theta)$ on $B'_{\mathrm{adv}}$ \Comment{outer minimization}
  \EndFor
\EndWhile
\State \textbf{Stage 2: AT-GRPO}
\For{iteration $j = 1, 2, \dots$}
  \State $\pi_{\theta_{\mathrm{old}}} \leftarrow \pi_\theta$
  \For{minibatch $B \subset \mathcal{D}_{\mathrm{RL}}$}
    \For{state $s \in B$} \Comment{inner maximization: minimize reward}
      \State Generate $\delta^*$ using PGD (Eq.~\eqref{eq:pgd_update}) with $\mathcal{L} = -J(\theta_{\mathrm{old}})$
      \State $s'_{\mathrm{adv}} \leftarrow s + \delta^*$
      \State Sample $\{Y_i\}_{i=1}^K \sim \pi_{\theta_{\mathrm{old}}}(\cdot \mid s'_{\mathrm{adv}})$
      \State Compute rewards $\{r_i\}$ and advantages $\{\hat{A}_i\}$
    \EndFor
  \EndFor
  \State Update $\theta$ by maximizing $\mathcal{L}^{\mathrm{GRPO}}(\theta)$ on adversarial trajectories \Comment{outer maximization}
\EndFor
\end{algorithmic}
\end{algorithm}

\subsection{Evaluation}

We assess both empirical robustness (against concrete attacks) and certified robustness (Randomized Smoothing).

\subsubsection{Adversarial threat landscape}

Given $(I, Q)$ and policy $\pi_{\theta^*}$, the adversary searches for $\delta_I$:
\begin{equation}
  \delta_I^* \in \Delta(I)
  \ \text{s.t.} \
  \text{ExtractA}(\pi_{\theta^*}(I+\delta_I^*, Q))
  \neq
  \text{ExtractA}(\pi_{\theta^*}(I, Q)).
\end{equation}

We consider several attack methods:

\begin{enumerate}
  \item \textbf{FGSM}:
  \begin{equation}
    \delta^{FGSM}
    = \epsilon \cdot \text{sign}(\nabla_I \mathcal{L}(\pi_{\theta^*}, I, Q)).
  \end{equation}
  \item \textbf{PGD}:
  \begin{equation}
    \delta^{k+1}
    = \Pi_{\Delta(I)} \left(
      \delta^k + \alpha \cdot \text{sign}(\nabla_I \mathcal{L}(\pi_{\theta^*}, I+\delta^k, Q))
    \right).
  \end{equation}
  \item \textbf{C\&W}:
  \begin{equation}
    \min_{\delta_I} \|\delta_I\|_2 + c \cdot f(I+\delta_I, Q),
  \end{equation}
  where $f$ promotes misclassification.
\end{enumerate}

\subsubsection{Certified robustness via Randomized Smoothing}

We define the smoothed policy $\Pi(I,Q)$ by adding Gaussian noise to $I$ and taking the most likely answer:
\begin{equation}
  \label{eq:smoothed_policy}
  \Pi(I, Q)
  = \arg\max_{A \in \mathcal{A}}
    P_{\eta \sim \mathcal{N}(0, \sigma^2 \mathbf{I})}
    \bigl(
      \text{ExtractA}(\pi_{\theta^*}(I+\eta, Q)) = A
    \bigr).
\end{equation}

\begin{definition}[Visually smoothed VQA policy $\Pi$]
\label{def:smoothed_policy}
The smoothed policy $\Pi(I, Q)$ returns the final answer $A$ that is statistically most likely to be chosen by the base policy when the image $I$ is perturbed by Gaussian noise $\eta$.
\end{definition}

\begin{theorem}[$L_2$ certified robustness (adapted from Cohen et al., 2019)]
\label{thm:l2_certified_robustness}
Let $\Pi$ be the smoothed policy (Eq.~\ref{eq:smoothed_policy}) with noise level $\sigma$. Suppose for input $(I, Q)$, the most probable answer is $A_A$ with probability $p_A$, and the runner-up probability is $p_B$. Then $\Pi$ is provably robust against any perturbation $\delta_I$ with $\|\delta_I\|_2 < R_{cert}$, where:
\begin{equation}
  \label{eq:certified_radius}
  R_{cert}
  = \frac{\sigma}{2} \left(\Phi^{-1}(p_A) - \Phi^{-1}(p_B)\right),
\end{equation}
and $\Phi^{-1}$ is the inverse CDF of the standard Gaussian.
\end{theorem}

We estimate $p_A$ via Monte Carlo and compute a high-confidence lower bound to obtain a certified radius. Algorithm~\ref{alg:rs_evaluation} summarizes the procedure.

\begin{algorithm}[H]
\caption{Certified robustness evaluation via Randomized Smoothing}
\label{alg:rs_evaluation}
\begin{algorithmic}[1]
\Require Input $(I, Q)$, base policy $\pi_{\theta^*}$, noise level $\sigma$, samples $N_{pred}, N_{cert}$, confidence $\alpha$.
\Ensure Smoothed prediction $A^*$, certified radius $R_{cert}$.
\State \textbf{1. Prediction (majority vote estimation):}
\State Sample $\{\eta_i\}_{i=1}^{N_{pred}} \sim \mathcal{N}(0, \sigma^2 \mathbf{I})$.
\State $A_i \leftarrow \text{ExtractA}(\pi_{\theta^*}(I+\eta_i, Q))$ for $i=1,\dots,N_{pred}$.
\State $A^* \leftarrow \arg\max_{A} \sum_{i=1}^{N_{pred}} \mathbb{I}(A_i = A)$.
\State \textbf{2. Certification (radius calculation):}
\State Sample $\{\eta'_j\}_{j=1}^{N_{cert}} \sim \mathcal{N}(0, \sigma^2 \mathbf{I})$.
\State $C_{A^*} \leftarrow \sum_{j=1}^{N_{cert}} \mathbb{I}(\text{ExtractA}(\pi_{\theta^*}(I+\eta'_j, Q)) = A^*)$.
\State $p_L \leftarrow \text{BinomialLowerBound}(C_{A^*}, N_{cert}, \alpha)$.
\If{$p_L > 0.5$}
  \State $R_{cert} \leftarrow \sigma \Phi^{-1}(p_L)$.
\Else
  \State $R_{cert} \leftarrow 0$.
\EndIf
\State \Return $A^*, R_{cert}$.
\end{algorithmic}
\end{algorithm}

\section{Results and Discussion}

We evaluate the SafeMed-R1 framework on a comprehensive benchmark of medical Visual Question Answering (VQA) tasks, focusing on both clinical accuracy and the robustness of the generated reasoning. We compare our fine-tuned models against an extensive set of state-of-the-art zero-shot VLMs (both general-purpose and medical-specific) under standard (clean) conditions and under strong adversarial attacks (PGD).

\subsection{Experimental Setup}

\paragraph{Evaluation Protocol.}
We assess models using the OmniMedVQA dataset, spanning eight distinct imaging modalities. Performance is measured by accuracy, defined as the percentage of correctly answered questions. For each input image–question pair, the model generates a response that is compared against the ground truth.

We conduct two primary evaluations:
\begin{enumerate}
  \item \textbf{Clean Evaluation:} The VLM receives the original, unperturbed medical image as input.
  \item \textbf{Adversarial Evaluation (PGD):} The VLM receives an image perturbed using the Projected Gradient Descent (PGD) algorithm, simulating a strong first-order adversarial attack on the visual input.
\end{enumerate}

\paragraph{VQA Format and Model Variants.}
We evaluate two primary configurations of our fine-tuned Qwen3-VL-4B models, corresponding to different stages and objectives of the training pipeline:
\begin{itemize}
  \item \textbf{Instruct (Instr.):} Models optimized via SFT (or AT-SFT) to directly generate the final answer within \verb|<answer>| tags.
  \item \textbf{Thinking (Think):} Models optimized via SFT+GRPO (or AT-SFT+AT-GRPO) to generate an explicit Chain-of-Thought (CoT) reasoning trace within \verb|<think>| tags, followed by the final answer. The GRPO stage specifically targets the quality of this reasoning.
\end{itemize}
Figure~\ref{fig:vqa_examples} illustrates the structured input and corresponding outputs for both variants.

\begin{figure}[t]
  \centering
  \includegraphics[width=0.8\textwidth]{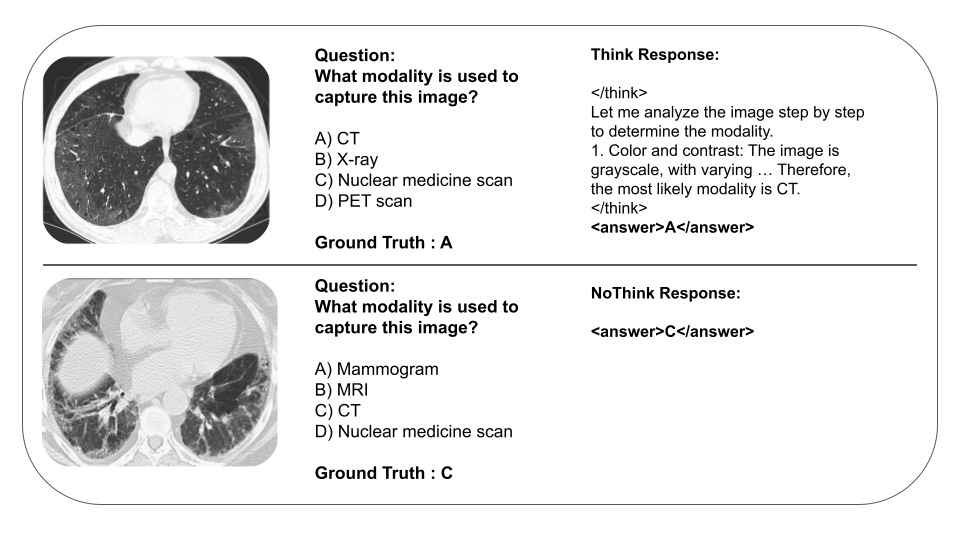}
  \caption{Examples of Medical VQA inputs and the corresponding responses from the SafeMed-R1 \textbf{Think} (top) and \textbf{Instruct} (bottom) configurations. The Think model generates explicit reasoning steps before concluding the answer.}
  \label{fig:vqa_examples}
\end{figure}

\subsection{Performance under Clean Evaluation}

Table~\ref{tab:vlm-modality-pgd} details the performance when models are subjected to PGD attacks on the input images. This evaluation reveals the true resilience of the models. We observe a substantial performance gap between zero-shot models and our fine-tuned models. The best zero-shot general VLM (Qwen3-VL-4B) reaches 78.00\% overall accuracy, and the best medical VLM (MedGemma-4B) reaches 68.75\%. In contrast, our standard fine-tuned models achieve 95.43\% (Instruct) and 94.25\% (Think) overall accuracy. This highlights the necessity of domain-specific fine-tuning for high-stakes applications such as medical VQA.

\begin{table}[h]
  \caption{Performance (Accuracy \%) of general-purpose, medical, fine-tuned (Clean FT), and adversarially fine-tuned (Adv FT) VLMs across imaging modalities under \textbf{clean evaluation}. The highest accuracy within the fine-tuned categories (Clean FT and Adv FT) for each modality and overall is highlighted in \textcolor{red}{\textbf{red}}.}
  \label{tab:vlm-modality-clean}
  \centering
  \setlength{\tabcolsep}{3pt}
  \resizebox{\columnwidth}{!}{%
    \begin{tabular}{lccccccccc}
      \toprule
      Methods & \multicolumn{8}{c}{Modality} & Overall \\
      \cmidrule(lr){2-9}
              & CT & MRI & X-Ray & Ultrasound & Dermoscopy & Fundus & OCT & Microscopy & \\
      \midrule
      \multicolumn{10}{c}{\textbf{Zero-shot VLMs (General Purpose)}} \\
      \midrule
      BLIP-2               & 56.74 & 41.32 & 67.58 & 37.27 & 40.65 & 46.24 & 68.08 & 50.40 & 51.04 \\
      InstructBLIP         & 28.72 & 33.15 & 61.04 & 41.25 & 62.22 & 50.31 & 42.59 & 46.29 & 45.70 \\
      LLaVA                & 17.73 & 26.72 & 30.70 & 18.66 & 49.74 & 47.11 & 33.73 & 28.87 & 31.66 \\
      LLaMA Adapter v2     & 21.41 & 26.63 & 46.44 & 34.05 & 51.76 & 50.74 & 33.00 & 38.66 & 37.83 \\
      MiniGPT-4            & 22.81 & 27.48 & 38.30 & 25.50 & 40.25 & 38.33 & 31.40 & 36.23 & 32.54 \\
      InternVL2            & 40.20 & 58.10 & 57.90 & 49.10 & 51.90 & 53.20 & 59.10 & 64.00 & 54.19 \\
      Qwen2-VL-2B          & 45.10 & 38.57 & 39.32 & 30.86 & 35.83 & 43.17 & 35.14 & 36.85 & 38.11 \\
      Qwen2-VL-7B          & 61.46 & 45.77 & 64.27 & 36.01 & 49.08 & 59.84 & 59.32 & 61.08 & 54.60 \\
      Qwen2-VL-72B         & 67.97 & 69.39 & 77.21 & 51.39 & 65.31 & 72.58 & 72.76 & 67.83 & 68.05 \\
      Qwen2.5-VL-3B        & 59.00 & 64.00 & 79.00 & 33.00 & 64.00 & 70.00 & 76.00 & 63.00 & 63.50 \\
      Qwen2.5-VL-7B        & 60.00 & 67.00 & 73.00 & 33.00 & 65.00 & 62.00 & 64.00 & 68.00 & 61.50 \\
      Qwen2.5-VL-72B       & 66.18 & 68.74 & 77.59 & 49.81 & 69.75 & 71.04 & 69.22 & 69.37 & 67.71 \\
      Qwen3-VL-4B          & 58.00 & 76.00 & 89.00 & 80.00 & 78.00 & 87.00 & 72.00 & 84.00 & 78.00 \\
      Qwen3-VL-8B          & 63.00 & 77.00 & 85.00 & 62.00 & 77.00 & 82.00 & 75.00 & 79.00 & 75.00 \\
      \midrule
      \multicolumn{10}{c}{\textbf{Zero-shot Medical VLMs}} \\
      \midrule
      LLaVA-Med            & 18.69 & 27.47 & 30.68 & 29.88 & 44.95 & 39.03 & 34.61 & 33.29 & 32.33 \\
      RadFM                & 27.56 & 24.06 & 30.95 & 16.57 & 39.21 & 36.89 & 32.80 & 27.97 & 29.50 \\
      Med-Flamingo         & 38.47 & 40.56 & 30.34 & 24.64 & 32.43 & 30.12 & 26.51 & 19.93 & 30.38 \\
      MedVInT              & 40.74 & 43.10 & 55.10 & 41.26 & 29.11 & 31.84 & 23.26 & 32.00 & 37.05 \\
      HuatouGPT-Vision     & 35.30 & 40.40 & 41.50 & 60.10 & 53.10 & 51.40 & 59.30 & 62.30 & 50.43 \\
      HealthGPT            & 35.50 & 78.50 & 81.90 & 51.40 & 64.90 & 54.60 & 89.30 & 88.20 & 68.04 \\
      MedGemma-4B          & 66.00 & 62.00 & 75.00 & 53.00 & 74.00 & 79.00 & 76.00 & 65.00 & 68.75 \\
      \midrule
      \multicolumn{10}{c}{\textbf{Fine-tuned VLMs (Clean FT)}} \\
      \midrule
      Qwen2-VL-2B (SFT)            & 51.74 & 52.83 & 65.57 & 47.65 & 51.91 & 52.26 & 53.99 & 56.58 & 54.07 \\
      \textbf{Qwen3-VL-4B Instr.}  & \textcolor{red}{\textbf{96.62}} & 92.72 & 92.16 & 95.24 & \textcolor{red}{\textbf{96.83}} & 94.02 & \textcolor{red}{\textbf{99.32}} & 96.51 & \textcolor{red}{\textbf{95.43}} \\
      \textbf{Qwen3-VL-4B Think}   & 94.57 & \textcolor{red}{\textbf{93.73}} & \textcolor{red}{\textbf{93.26}} & \textcolor{red}{\textbf{96.93}} & 94.17 & 90.57 & 93.78 & \textcolor{red}{\textbf{97.01}} & 94.25 \\
      \midrule
      \multicolumn{10}{c}{\textbf{Adversarial Fine-tuned VLMs (Adv FT -- SafeMed-R1)}} \\
      \midrule
      \textbf{Qwen3-VL-4B Instr.}  & 89.64 & 89.58 & 92.78 & 92.79 & 93.71 & \textcolor{red}{\textbf{95.64}} & 94.52 & 93.50 & 92.77 \\
      \textbf{Qwen3-VL-4B Think}   & 94.50 & 90.91 & 91.26 & 90.05 & 90.64 & 93.68 & 88.83 & 91.90 & 91.47 \\
      \bottomrule
    \end{tabular}%
  }
\end{table}

Robustness--accuracy trade-off, consistent with prior work on adversarial training, we see a trade-off between clean accuracy and robustness. Adversarially fine-tuned models (SafeMed-R1) show slightly lower clean accuracy (92.77\% Instr., 91.47\% Think) than standard fine-tuned models (95.43\% and 94.25\%). The drop (e.g., 95.43\% $\rightarrow$ 92.77\% for Instruct) is expected when parameters are optimized to handle worst-case perturbations, which tends to yield smoother decision boundaries that can misclassify some clean inputs near the boundary. Impact of reasoning (Think vs.\ Instruct) under clean evaluation, Instruct models (SFT only) slightly outperform Think models (SFT+GRPO) overall (95.43\% vs.\ 94.25\%). GRPO targets reasoning quality rather than pure answer accuracy; in some modalities (e.g., CT, Dermoscopy, OCT), the explicit reasoning step slightly reduces performance compared to the best Instruct model. This suggests a mild misalignment: optimizing for reasoning coherence (as scored by the reward model) does not always coincide with maximizing answer accuracy on clean multiple-choice questions.

\subsection{Performance under Adversarial Evaluation (PGD Attack)}

Table~\ref{tab:vlm-modality-pgd} details the performance when models are subjected to PGD attacks on the input images. This evaluation reveals the true resilience of the models. There is vulnerability of standard models. Despite strong clean performance, standard fine-tuned models collapse under PGD attack: overall accuracy drops to roughly 25--26\%. This shows that high scores on standard benchmarks do not guarantee robustness and can be dangerously misleading in safety-critical deployments.

\begin{table}[H]
  \caption{Performance (Accuracy \%) of VLMs across imaging modalities under \textbf{PGD-attack evaluation}. Adversarial fine-tuning drastically improves robustness compared to standard fine-tuning. The highest robust accuracy within the fine-tuned categories is highlighted in \textcolor{red}{\textbf{red}}.}
  \label{tab:vlm-modality-pgd}
  \centering
  \setlength{\tabcolsep}{3pt}
  \resizebox{\columnwidth}{!}{%
    \begin{tabular}{lccccccccc}
      \toprule
      Methods & \multicolumn{8}{c}{Modality} & Overall \\
      \cmidrule(lr){2-9}
              & CT & MRI & X-Ray & Ultrasound & Dermoscopy & Fundus & OCT & Microscopy & \\
      \midrule
      \multicolumn{10}{c}{\textbf{Zero-shot VLMs (General Purpose)}} \\
      \midrule
      Qwen2.5-VL-3B        & 53.00 & 53.00 & 56.00 & 29.00 & 48.00 & 63.00 & 67.00 & 51.00 & 52.50 \\
      Qwen2.5-VL-7B        & 57.00 & 58.00 & 59.00 & 30.00 & 55.00 & 52.00 & 50.00 & 60.00 & 52.63 \\
      Qwen3-VL-4B          & 56.00 & 61.00 & 54.00 & 53.00 & 56.00 & 70.00 & 69.00 & 57.00 & 59.50 \\
      Qwen3-VL-8B          & 58.00 & 63.00 & 56.00 & 43.00 & 57.00 & 67.00 & 67.00 & 51.00 & 57.75 \\
      \midrule
      \multicolumn{10}{c}{\textbf{Zero-shot Medical VLMs}} \\
      \midrule
      MedGemma-4B          & 52.00 & 48.00 & 60.00 & 33.00 & 60.00 & 69.00 & 59.00 & 54.00 & 54.38 \\
      \midrule
      \multicolumn{10}{c}{\textbf{Fine-tuned VLMs (Clean FT)}} \\
      \midrule
      \textbf{Qwen3-VL-4B Instr.} & 23.62 & 24.87 & 23.52 & 21.25 & 23.53 & 32.87 & 33.88 & 25.04 & 26.07 \\
      \textbf{Qwen3-VL-4B Think}  & 27.48 & 17.31 & 21.35 & 23.30 & 32.33 & 20.01 & 24.66 & 34.71 & 25.14 \\
      \midrule
      \multicolumn{10}{c}{\textbf{Adversarial Fine-tuned VLMs (Adv FT -- SafeMed-R1)}} \\
      \midrule
      \textbf{Qwen3-VL-4B Instr.} & \textcolor{red}{\textbf{82.79}} & 84.19 & 76.81 & \textcolor{red}{\textbf{87.40}} & 84.05 & 83.18 & 80.14 & 79.56 & 82.27 \\
      \textbf{Qwen3-VL-4B Think}  & 81.26 & \textcolor{red}{\textbf{85.22}} & \textcolor{red}{\textbf{88.13}} & 82.66 & \textcolor{red}{\textbf{85.04}} & \textcolor{red}{\textbf{83.79}} & \textcolor{red}{\textbf{84.37}} & \textcolor{red}{\textbf{85.12}} & \textcolor{red}{\textbf{84.45}} \\
      \bottomrule
    \end{tabular}%
  }
\end{table}

Observing the efficacy of SafeMed-R1 we can see that in contrast, models trained with the SafeMed-R1 framework maintain high performance under attack. Adv FT–Think reaches 84.45\% accuracy and Adv FT–Instruct reaches 82.27\%. This corresponds to an absolute gain of about 58 percentage points over clean-only fine-tuning. The drop from their clean accuracies (e.g., 91.47\% $\rightarrow$ 84.45\% for Adv FT–Think) is modest, showing that the combined AT-SFT and AT-GRPO procedure effectively hardens the model against strong first-order adversarial perturbations.

\paragraph{Reasoning enhances robustness.}
Under attack, the trend between Instruct and Think reverses. Adv FT–Think outperforms Adv FT–Instruct (84.45\% vs.\ 82.27\%). This suggests that explicit, GRPO-optimized reasoning contributes positively to robustness: when inputs are perturbed, a structured multi-step reasoning process appears to help the model compensate for misleading visual cues and still reach the correct conclusion more often than a direct instruction-following policy.

\subsection{Limitations}

SafeMed-R1 substantially improves empirical robustness, but our evaluation is limited to PGD-based attacks on the visual modality. Due to computational constraints, we did not exhaustively test other attack families, such as optimization-based C\&W attacks, multimodal or text-only attacks, or adaptive strategies targeting the reasoning process itself. Furthermore, the evaluation of certified robustness via randomized smoothing (as described in the methodology) remains computationally expensive for large VLMs, and scaling certified defenses to realistic medical workloads remains an open problem.

\section{Conclusion}

We introduced SafeMed-R1, a framework that integrates adversarial training with Group Relative Policy Optimization (GRPO) to jointly enhance reasoning quality and adversarial robustness in medical Vision–Language Models. Our experiments show that standard fine-tuning can achieve high accuracy on clean data while remaining extremely vulnerable to adversarial perturbations. In contrast, SafeMed-R1 substantially mitigates this vulnerability, preserving high accuracy even under PGD attack. Moreover, adversarially trained Think models benefit from explicit reasoning, yielding both interpretable chains of thought and improved robustness. These findings point toward a useful synergy between interpretability and security in the design of trustworthy medical AI systems.

\begin{ack}
We would like to acknowledge that computational work involved in this research is partially/fully supported by NUS IT’s Research Computing group under grant number NUSREC-HPC-00001.
\end{ack}

\newpage

{
\bibliographystyle{plainnat}
\bibliography{references}
}


\newpage
\appendix

\numberwithin{equation}{section}
\numberwithin{theorem}{section}
\numberwithin{definition}{section}
\numberwithin{assumption}{section}

\section{Convergence Analysis}

This appendix provides a theoretical analysis of the convergence properties of the optimization strategies employed in the SafeMed-R1 framework. We analyze convergence for both Stage 1 (Supervised Fine-Tuning) and Stage 2 (Group Relative Policy Optimization), including their adversarial counterparts (AT-SFT and AT-GRPO), and the statistical guarantees of the inference-time Randomized Smoothing evaluation.

\subsection{Preliminaries and Assumptions}

We begin by establishing the necessary assumptions regarding the optimization landscape, which are standard in the analysis of non-convex optimization and reinforcement learning \cite{kingma2015adam, schulman2017proximal}.

\begin{assumption}[L-smoothness]
\label{asm:smoothness}
The loss functions, specifically the SFT loss $\mathcal{L}_{SFT}(\theta)$ (Eq.~\ref{eq:sft_loss}) and the negative expected reward $-J(\theta)$ (Eq.~\ref{eq:rl_objective}), are L-smooth. That is, their gradients are Lipschitz continuous:
\begin{equation}
  \|\nabla \mathcal{L}(\theta_1) - \nabla \mathcal{L}(\theta_2)\|
  \leq L \|\theta_1 - \theta_2\|.
\end{equation}
\end{assumption}

\begin{assumption}[Bounded variance and gradients]
\label{asm:variance}
The stochastic gradients computed during optimization have bounded variance ($\sigma^2$), and the norm of the gradients is bounded by a constant $G$.
\end{assumption}

\begin{assumption}[Bounded rewards]
\label{asm:rewards}
The scalar rewards generated by the reward model $R_\phi$ (Definition~\ref{def:reasoning_quality_reward}) are bounded within a finite range $[R_{\min}, R_{\max}]$.
\end{assumption}

\subsection{Convergence of Stage 1: SFT and AT-SFT}

Stage 1 optimizes the VLM parameters $\theta$ using the cross-entropy loss $\mathcal{L}_{SFT}(\theta)$, with the AdamW optimizer.

\subsubsection{Supervised Fine-Tuning (SFT)}

\begin{theorem}[Convergence of SFT]
\label{thm:sft_convergence}
Under Assumptions~\ref{asm:smoothness} and \ref{asm:variance}, and given an appropriately decaying learning rate schedule (e.g., cosine schedule, see Appendix~B.10), the AdamW optimizer guarantees convergence to a stationary point $\theta_{SFT}$. The convergence rate for non-convex objectives is:
\begin{equation}
  \min_{t=1,\dots,T}
  \mathbb{E}[\|\nabla \mathcal{L}_{SFT}(\theta_t)\|^2]
  = O\!\left(\frac{1}{\sqrt{T}}\right).
\end{equation}
\end{theorem}

\subsubsection{Adversarial Supervised Fine-Tuning (AT-SFT)}

AT-SFT involves solving a robust optimization problem formulated as a minimax game (Eq.~\ref{eq:minimax_obj}):
\begin{equation}
  \theta^{*}
  = \arg\min_{\theta}
    \mathbb{E}_{s} \left[
      \max_{\delta \in \Delta(s)} \mathcal{L}_{SFT}(\theta, s+\delta)
    \right].
\end{equation}
The inner maximization is approximated via PGD.

\begin{lemma}[Inner maximization convergence]
\label{lem:pgd_convergence}
For a fixed $\theta$, PGD (Eq.~\ref{eq:pgd_update}) converges to a local maximum of the adversarial loss $\mathcal{L}_{SFT}(\theta, s+\delta)$ within the threat boundary $\Delta(s)$, providing a strong approximation of the worst-case perturbation $\delta^*$ \cite{madry2018towards}.
\end{lemma}

Using Danskin's theorem, if the inner maximization is solved (approximately), the gradient of the robust objective is given by the gradient at $\delta^*$.

\begin{theorem}[Convergence of AT-SFT]
\label{thm:atsft_convergence}
While convergence to a global saddle point is generally intractable due to non-convexity, AT-SFT converges to a robust local minimum. The optimization dynamics ensure that the model parameters are optimized against the strongest first-order adversary (PGD) approximated during training \cite{madry2018towards}.
\end{theorem}

\subsection{Convergence of Stage 2: GRPO and AT-GRPO}

Stage 2 employs GRPO (Section 3.2) to maximize the expected reward $J(\theta)$ (Eq.~\ref{eq:rl_objective}). GRPO is a trust-region-style method, algorithmically similar to PPO \cite{schulman2017proximal} and derived from TRPO \cite{schulman2015trust}. It utilizes a clipped surrogate objective (Eq.~\ref{eq:grpo_loss}) and a group-normalized advantage estimator (Eq.~\ref{eq:grpo_advantage}).

\subsubsection{Monotonic improvement guarantee}

Trust-region methods relate the improvement between the new policy $\pi_{\theta_{k+1}}$ and the old policy $\pi_{\theta_k}$ to a surrogate objective and the KL divergence between the policies:
\begin{equation}
  J(\theta_{k+1})
  \geq J(\theta_k)
    + \mathbb{E}_{\pi_{\theta_k}}
      \left[\frac{\pi_{\theta_{k+1}}}{\pi_{\theta_k}} A_{\pi_{\theta_k}}\right]
    - C \cdot D_{KL}^{\max}(\pi_{\theta_{k+1}} \,\|\, \pi_{\theta_k}),
\end{equation}
where $A$ is the advantage function and $C$ is a penalty coefficient. The clipping mechanism in GRPO (Eq.~\ref{eq:grpo_loss}) implicitly bounds the KL divergence, stabilizing updates \cite{schulman2017proximal}.

\subsubsection{Variance reduction via group normalization}

\begin{lemma}[Variance reduction in GRPO]
\label{lem:grpo_variance}
The group-normalized advantage estimator $\hat{A}_i$ (Eq.~\ref{eq:grpo_advantage}) acts as an implicit, localized baseline that significantly reduces the variance of the policy gradient estimates (Eq.~\ref{eq:policy_gradient}) without requiring a learned critic \cite{lai2025medr1reinforcementlearninggeneralizable, shao2024deepseekmathpushinglimitsmathematical}.
\end{lemma}

\begin{proof}
For a given state $s$, GRPO samples $K$ responses (Eq.~\ref{eq:grpo_sampling}) and computes the empirical mean $\bar{r}$ and variance $\text{Var}(r)$ (Eq.~\ref{eq:grpo_stats}). The normalization step (Eq.~\ref{eq:grpo_advantage}) centers the advantages within the group, with $\mathbb{E}_{i \in K}[\hat{A}_i] \approx 0$. This local centering reduces gradient magnitudes and hence their variance, improving the signal-to-noise ratio.
\end{proof}

\begin{theorem}[Convergence of GRPO]
\label{thm:grpo_convergence}
Combining the monotonic improvement guarantees of trust-region methods (via clipping) with the low-variance advantage estimation (Lemma~\ref{lem:grpo_variance}), GRPO ensures stable and monotonically improving expected rewards $J(\theta)$. The algorithm converges to a local optimum of the policy $\pi_{\theta^*}$.
\end{theorem}

\subsubsection{Convergence of Adversarial GRPO (AT-GRPO)}

AT-GRPO extends robust optimization to the RL setting, aiming to maximize:
\begin{equation}
  J_{\text{Rob}}(\theta)
  = \mathbb{E}_{s}\left[
      \min_{\delta \in \Delta(s)}
      \mathbb{E}_{Y \sim \pi_\theta(Y \mid s+\delta)}
      [r(Y, s+\delta)]
    \right].
\end{equation}
The inner minimization is approximated via PGD (Algorithm~\ref{alg:adv_ft}), and the outer maximization uses GRPO on adversarially perturbed states.

\begin{theorem}[Convergence of AT-GRPO]
\label{thm:atgrpo_convergence}
AT-GRPO converges to a robust local optimum of the policy. The trust-region constraints (via clipping) ensure that the policy updates remain stable even on adversarial inputs, guaranteeing monotonic improvement of the robust expected reward $J_{\text{Rob}}(\theta)$.
\end{theorem}

\subsection{Statistical Guarantees of Randomized Smoothing Estimation}

Randomized Smoothing provides certified robustness guarantees based on Monte Carlo estimates of answer probabilities under Gaussian noise.

\begin{definition}[Monte Carlo estimation of $p_A$]
\label{def:mc_estimation_pa}
The probability $p_A$ that the base policy returns answer $A$ under noise $\sigma$ is estimated by:
\begin{equation}
  \hat{p}_A
  = \frac{1}{N} \sum_{i=1}^N \mathbb{I}(\text{ExtractA}(\pi_{\theta^*}(I+\eta_i, Q)) = A),
\end{equation}
where $\eta_i \sim \mathcal{N}(0, \sigma^2 \mathbf{I})$.
\end{definition}

\begin{theorem}[Concentration bound for RS estimation]
\label{thm:concentration_bound_rs}
By the Law of Large Numbers, $\hat{p}_A \to p_A$ as $N \to \infty$. Moreover, by Hoeffding’s inequality:
\begin{equation}
  P(|\hat{p}_A - p_A| \geq \epsilon)
  \leq 2 \exp(-2N\epsilon^2),
\end{equation}
so the estimation error decays at rate $O(1/\sqrt{N})$.
\end{theorem}

We obtain a high-confidence lower bound $p_L$ using a binomial confidence interval (e.g., Clopper--Pearson).

\begin{lemma}[Reliability of certification]
\label{lem:reliability_certification}
The certified radius $R_{cert}$ computed in Algorithm~\ref{alg:rs_evaluation} holds with probability at least $1-\alpha$.
\end{lemma}

\begin{proof}
We ensure $P(p_A \geq p_L) \geq 1-\alpha$ by construction of the lower bound. When $p_L > 0.5$, the majority vote is correct with probability at least $1-\alpha$, and the radius $R_{cert} = \sigma \Phi^{-1}(p_L)$ (assuming $p_B = 1 - p_L$ for tightest bound) is a valid certified radius with confidence $1-\alpha$.
\end{proof}

\newpage

\section{Experiment Setup}
\label{sec:experiment_setup}

\subsection{Overview}

We present SafeMed-R1, a comprehensive training pipeline for robust medical vision--language models. Our experiments encompass four distinct training paradigms across two model variants (Thinking and Instruct), evaluated on eight medical imaging modalities.

\subsection{Training Infrastructure}

All experiments were conducted on a high-performance computing cluster with the following specifications:
\begin{itemize}
  \item \textbf{GPUs}: 16$\times$ NVIDIA H200 (143GB HBM3e memory per GPU)
  \item \textbf{Total GPU memory}: 2.288TB
  \item \textbf{Interconnect}: NVLink and InfiniBand for high-speed inter-GPU communication
  \item \textbf{Configuration}: 2 nodes with 8 GPUs each
  \item \textbf{Training strategy}: Fully Sharded Data Parallel (FSDP) with gradient checkpointing
\end{itemize}

\subsubsection{Software Environment}

We use a modern PyTorch stack with mixed-precision and fused kernels for efficiency:
\begin{itemize}
  \item PyTorch 2.1.0 with CUDA 12.1
  \item Mixed-precision training (bfloat16)
  \item TensorFloat-32 (TF32) acceleration
  \item FlashAttention-2 for efficient attention computation
  \item \texttt{accelerate} library for distributed training orchestration
  \item Unsloth framework for optimized vision--language model training
\end{itemize}

\subsection{Model Architectures}

\subsubsection{Base Models}

We employ two variants of the Qwen-VL architecture family:
\begin{itemize}
  \item \textbf{Qwen3-VL-4B-Instruct}: instruction-tuned variant for direct response generation (4B parameters)
  \item \textbf{Qwen3-VL-4B-Thinking}: chain-of-thought variant with explicit reasoning steps (4B parameters)
\end{itemize}

Both models use a vision transformer backbone with cross-modal attention, enabling effective vision--language alignment for medical imaging tasks.

\subsubsection{Parameter-Efficient Fine-Tuning}

We apply Low-Rank Adaptation (LoRA) for memory-efficient fine-tuning across all stages. Table~\ref{tab:lora_config} summarizes the LoRA configurations used in SFT and GRPO stages.

\begin{table}[H]
\centering
\caption{LoRA configurations for different training stages.}
\label{tab:lora_config}
\begin{tabular}{lcc}
\toprule
\textbf{Parameter} & \textbf{SFT/AT-SFT} & \textbf{GRPO/AT-GRPO} \\
\midrule
LoRA rank ($r$) & 64 & 128 \\
LoRA alpha ($\alpha$) & 128 & 256 \\
LoRA dropout & 0.05 & 0.05 \\
Trainable parameters & 42.5M & 85.0M \\
\% of total parameters & 0.57\% & 1.14\% \\
Target modules & \multicolumn{2}{c}{\texttt{q\_proj, v\_proj, k\_proj, o\_proj,}} \\
& \multicolumn{2}{c}{\texttt{gate\_proj, up\_proj, down\_proj}} \\
\bottomrule
\end{tabular}
\end{table}

\subsection{Distributed Training Configuration}

We train across 16 H200 GPUs with fully sharded model states. Table~\ref{tab:distributed_config} details the distributed training setup.

\begin{table}[H]
\centering
\caption{Distributed training setup across 16 H200 GPUs.}
\label{tab:distributed_config}
\begin{tabular}{ll}
\toprule
\textbf{Configuration} & \textbf{Setting} \\
\midrule
Number of nodes & 2 \\
GPUs per node & 8 \\
Total GPUs & 16 \\
FSDP sharding strategy & \texttt{FULL\_SHARD} \\
FSDP backward prefetch & \texttt{BACKWARD\_PRE} \\
CPU offloading & Disabled (sufficient GPU memory) \\
Gradient checkpointing & Enabled \\
Communication backend & NCCL \\
Mixed precision & bfloat16 \\
\bottomrule
\end{tabular}
\end{table}

\subsection{Training Pipelines}

\subsubsection{Stage 1: Supervised Fine-Tuning (SFT)}

\paragraph{Non-Adversarial SFT.}

We first perform standard supervised fine-tuning on medical VQA datasets:
\begin{itemize}
  \item Learning rate: $2 \times 10^{-5}$
  \item Batch size per GPU: 4
  \item Gradient accumulation steps: 2
  \item Effective batch size: 128 (16 GPUs $\times$ 4 $\times$ 2)
  \item Training epochs: 3
  \item Optimizer: AdamW (8-bit quantized)
  \item Warmup steps: 500
  \item Evaluation frequency: every 500 steps
\end{itemize}

\paragraph{Adversarial Training SFT (AT-SFT).}

AT-SFT augments SFT with PGD-based perturbations on input embeddings:
\begin{itemize}
  \item Base configuration: same as non-adversarial SFT
  \item \textbf{Adversarial parameters:}
  \begin{itemize}
    \item Perturbation bound ($\epsilon$): 0.01
    \item PGD step size ($\alpha$): 0.002
    \item PGD iterations ($N_{\text{PGD}}$): 5
    \item Attack type: projected gradient descent
    \item Perturbation space: input embeddings
    \item Adversarial training ratio: 0.5 (alternating batches)
  \end{itemize}
\end{itemize}

\subsubsection{Stage 2: Group Relative Policy Optimization (GRPO)}

\paragraph{Non-Adversarial GRPO.}

We then apply GRPO to optimize response quality with AI feedback:
\begin{itemize}
  \item Learning rate: $1 \times 10^{-5}$
  \item Batch size per GPU: 2
  \item Gradient accumulation steps: 2
  \item Effective batch size: 64 (16 GPUs $\times$ 2 $\times$ 2)
  \item Training iterations: 1000
  \item KL coefficient ($\beta$): 0.05
  \item Reward model: ArmoRM-Llama3-8B-v0.1
  \item Reference model: frozen SFT checkpoint
\end{itemize}

\paragraph{Adversarial Training GRPO (AT-GRPO).}

AT-GRPO incorporates adversarial perturbations during RL fine-tuning:
\begin{itemize}
  \item Base configuration: same as non-adversarial GRPO
  \item \textbf{Adversarial parameters:}
  \begin{itemize}
    \item Perturbation bound ($\epsilon$): 0.01
    \item PGD step size ($\alpha$): 0.002
    \item PGD iterations ($N_{\text{PGD}}$): 5
    \item Adversarial reward weight: 0.3
    \item Robustness regularization: KL divergence between clean and adversarial outputs
  \end{itemize}
\end{itemize}

\subsection{Memory Optimization}

With 143GB per H200 GPU, we can support long sequences and large effective batch sizes while staying within safe utilization. Table~\ref{tab:memory_usage} reports the per-GPU memory budget across stages.

\begin{table}[H]
\centering
\caption{Memory utilization across training stages (per GPU).}
\label{tab:memory_usage}
\begin{tabular}{lcccc}
\toprule
\textbf{Component} & \textbf{SFT} & \textbf{AT-SFT} & \textbf{GRPO} & \textbf{AT-GRPO} \\
\midrule
Model weights    & 28GB  & 28GB  & 28GB  & 28GB  \\
LoRA adapters    & 0.5GB & 0.5GB & 1GB   & 1GB   \\
Optimizer states & 14GB  & 14GB  & 28GB  & 28GB  \\
Gradients        & 7GB   & 14GB  & 14GB  & 21GB  \\
Activations      & 35GB  & 45GB  & 50GB  & 60GB  \\
Batch data       & 15GB  & 15GB  & 20GB  & 20GB  \\
\midrule
\textbf{Total per GPU} & 99.5GB & 116.5GB & 121GB & 138GB \\
\textbf{Utilization}   & 69.6\% & 81.5\%  & 84.6\% & 96.5\% \\
\bottomrule
\end{tabular}
\end{table}

\subsection{Generation Configuration}

We keep generation settings fixed across GRPO stages to isolate training effects. Table~\ref{tab:generation_config} lists the decoding hyperparameters.

\begin{table}[H]
\centering
\caption{Generation parameters for GRPO stages.}
\label{tab:generation_config}
\begin{tabular}{lcc}
\toprule
\textbf{Parameter} & \textbf{Thinking Model} & \textbf{Instruct Model} \\
\midrule
Max prompt length      & 2048 & 1536 \\
Max completion length  & 1024 & 768  \\
Min length             & 1    & 1    \\
Temperature            & 0.7  & 0.7  \\
Top-$p$                & 0.9  & 0.9  \\
Top-$k$                & 50   & 50   \\
Repetition penalty     & 1.1  & 1.1  \\
Beam size (evaluation) & 4    & 4    \\
\bottomrule
\end{tabular}
\end{table}

\subsection{Training Schedule and Timeline}

Figure~\ref{fig:training_timeline} visualizes the wall-clock schedule for standard and adversarial pipelines per modality and across all eight modalities.

\begin{figure}[H]
\centering
\begin{tikzpicture}[scale=0.9, every node/.style={scale=0.9}]
\tikzstyle{stage} = [rectangle, draw, fill=blue!20, text width=2.5cm, text centered, rounded corners, minimum height=0.8cm]
\tikzstyle{arrow} = [thick,->,>=stealth]

\draw[thick] (0,0) -- (14,0);
\foreach \x in {0,2,4,6,8,10,12,14}
  \draw (\x,0.1) -- (\x,-0.1) node[below] {\x h};

\node[stage] at (3,1) (sft1) {SFT\\3 epochs\\48h};
\node[stage] at (9,1) (grpo1) {GRPO\\1000 iter\\96h};
\draw[arrow] (5,1) -- (7,1);

\node[stage] at (3,2.5) (atsft) {AT-SFT\\3 epochs\\72h};
\node[stage] at (9,2.5) (atgrpo) {AT-GRPO\\1000 iter\\120h};
\draw[arrow] (5,2.5) -- (7,2.5);

\node at (-1,1) {Standard:};
\node at (-1,2.5) {Adversarial:};

\node at (7,-1.5) {Total per modality: 144h (Standard), 192h (Adversarial)};
\node at (7,-2.0) {All 8 modalities: 1{,}152h (Standard), 1{,}536h (Adversarial)};
\end{tikzpicture}
\caption{Complete training pipeline timeline for SafeMed-R1.}
\label{fig:training_timeline}
\end{figure}

\subsection{Computational Efficiency}

We benchmark throughput and communication overhead against an 8$\times$H100 baseline. Table~\ref{tab:efficiency} reports key efficiency metrics.

\begin{table}[H]
\centering
\caption{Training efficiency metrics with 16$\times$ H200 GPUs.}
\label{tab:efficiency}
\begin{tabular}{lcc}
\toprule
\textbf{Metric} & \textbf{Value} & \textbf{vs 8$\times$ H100} \\
\midrule
FLOPs utilization    & 68\%   & +15\% \\
Samples per second   & 256    & 2.1$\times$ \\
Gradient sync time   & 0.8s   & $-45$\% \\
Checkpoint save time & 12s    & $-30$\% \\
Total training time  & 2{,}688h & $-40$\% \\
Energy consumption   & 4{,}300 kWh & $-25$\% \\
\bottomrule
\end{tabular}
\end{table}

\subsection{Hyperparameter Details}

For completeness, Table~\ref{tab:detailed_hyperparams} lists the core optimization hyperparameters shared across stages.

\begin{table}[H]
\centering
\caption{Detailed hyperparameters across all training stages.}
\label{tab:detailed_hyperparams}
\begin{tabular}{lcccc}
\toprule
\textbf{Hyperparameter} & \textbf{SFT} & \textbf{AT-SFT} & \textbf{GRPO} & \textbf{AT-GRPO} \\
\midrule
Learning rate      & $2\times 10^{-5}$ & $2\times 10^{-5}$ & $1\times 10^{-5}$ & $1\times 10^{-5}$ \\
Optimizer          & AdamW-8bit       & AdamW-8bit        & AdamW-8bit        & AdamW-8bit        \\
Adam $\beta_1$     & 0.9              & 0.9               & 0.9               & 0.9               \\
Adam $\beta_2$     & 0.999            & 0.999             & 0.95              & 0.95              \\
Weight decay       & 0.01             & 0.01              & 0.01              & 0.01              \\
Gradient clipping  & 1.0              & 1.0               & 1.0               & 1.0               \\
Warmup ratio       & 0.1              & 0.1               & 0.1               & 0.1               \\
LR scheduler       & Cosine           & Cosine            & Linear            & Linear            \\
\bottomrule
\end{tabular}
\end{table}

\end{document}